\documentclass[letterpaper, 10 pt, conference]{ieeeconf}  % Comment this line out if you need a4paper

\IEEEoverridecommandlockouts                              % This command is only needed if 
\usepackage{graphics} % for pdf, bitmapped graphics files
\usepackage{epsfig} % for postscript graphics files
\usepackage{mathptmx} % assumes new font selection scheme installed
\usepackage{times} % assumes new font selection scheme installed
\usepackage{amsmath} % assumes amsmath package installed
\usepackage{amssymb}  % assumes amsmath package installed
\usepackage{url}
\usepackage{cite}
\newtheorem{assumption}{Assumption}
\newtheorem{lemma}{Lemma}
\newtheorem{definition}{Definition}
\newtheorem{theorem}{Theorem}
\newtheorem{rem}{Remark}
\title{\LARGE \bf
Robust Operational Space Control with Conformal Disturbance Bounds for Safe Redundant Manipulation
}

\author{Wenhua Liu$^{1}$, Fan Zhang$^{1}$, Qin Lin$^{1}$ % <-this % stops a space
\thanks{$^{1}$These authors are with the Department of Engineering Technology and the Department of Electrical and Computer Engineering, University of Houston, Houston, TX 77004, USA. {\tt\small \{wliu40, fzhang28\}@cougarnet.uh.edu; qlin12@uh.edu}}
\thanks{This material is based upon work supported by the National Science Foundation under Grants No. 2301543.}
}

\begin{document}

\maketitle
\thispagestyle{empty}
\pagestyle{empty}

%%%%%%%%%%%%%%%%%%%%%%%%%%%%%%%%%%%%%%%%%%%%%%%%%%%%%%%%%%%%%%%%%%%%%%%%%%%%%%%%
\begin{abstract}
Redundant robotic manipulators operating in constrained and human-interactive environments require accurate task-space tracking together with rigorous safety guarantees under dynamic uncertainties. Classical operational space computed torque controller (OSCTC) relies on accurate dynamic models and degrades in the presence of disturbances. In contrast, the data-driven paradigm of residual learning approximates disturbances as functions learned from full-state measurements, which are often noisy in practice, lack rigorous theoretical guarantees, and introduce additional design complexity. This paper proposes a robust OSCTC framework that integrates an extended state observer (ESO) with conformal prediction to combine model-based robustness and data-driven adaptability. The ESO estimates lumped disturbances directly in operational space without requiring full-state measurements as in residual learning, and a robust control barrier function (CBF) is constructed to enforce safety under uncertainty. However, robust CBFs require a known disturbance-variation bound to guarantee absolute safety, which often leads to conservatism in practice. To address this limitation, we further employ a sliding-window conformal prediction mechanism to estimate the bound online in a distribution-free manner, thereby achieving practical probabilistic safety guarantees. Experiments on a 7-DoF Franka Research 3 manipulator demonstrate millimeter-level tracking accuracy and real-time safe control at 1~kHz under various disturbances. Video: {\url{https://www.youtube.com/watch?v=KHs64uKjZ1w}}.
\end{abstract}

%%%%%%%%%%%%%%%%%%%%%%%%%%%%%%%%%%%%%%%%%%%%%%%%%%%%%%%%%%%%%%%%%%%%%%%%%%%%%%%%
\section{Introduction}
Redundant robotic manipulators, equipped with additional degrees of freedom, are increasingly deployed in constrained and human-interactive environments due to their enhanced dexterity and flexibility. By exploiting kinematic redundancy, they can simultaneously achieve obstacle avoidance, posture optimization, and safe human-robot collaboration.

For robotic manipulators, control strategies are typically formulated either in joint space or directly in operational (task) space. Joint-space control requires mapping desired task-space trajectories to joint motions through \emph{inverse kinematics} (IK). For redundant manipulators, this mapping is not unique and requires additional redundancy resolution, which complicates the coordination of task-space objectives and constraints. In contrast, \emph{operational space computed torque controller (OSCTC)} \cite{Khatib1987OSC} enables direct regulation of end-effector motion in task space without solving IK, while preserving null-space freedom for secondary objectives such as posture optimization.

\begin{figure}[t]
    \centering
    \includegraphics[width=0.8\linewidth]{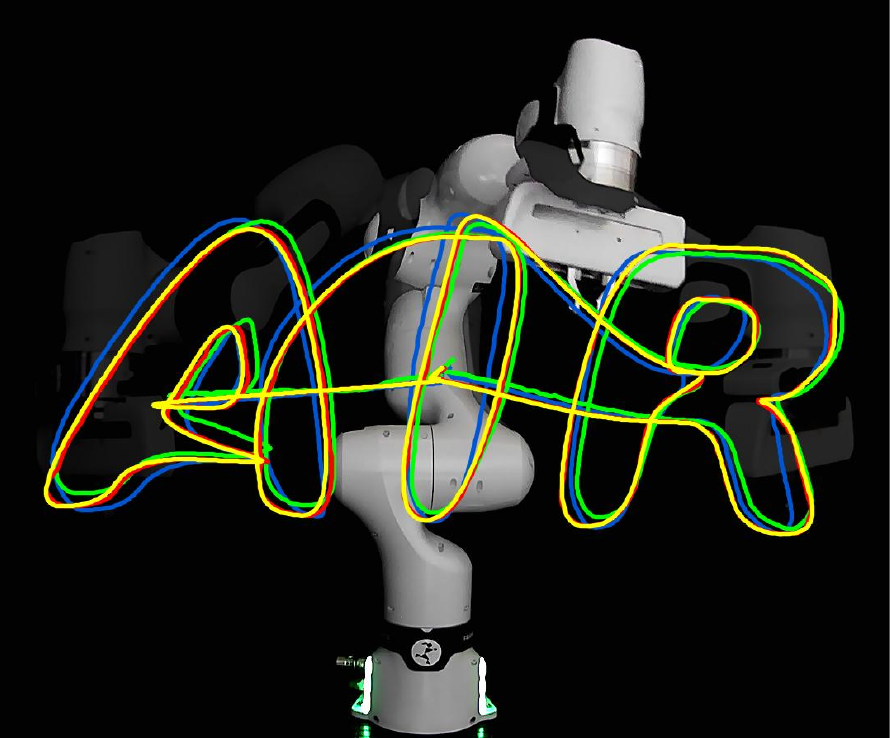}
    \caption{Task-space trajectory forming the letters “AIR” under model uncertainty. Yellow: reference trajectory. Blue: conventional OSCTC without disturbance compensation \cite{Khatib1987OSC}. Green: Residual learning. Red: proposed method, which has the minimum deviations caused by model mismatch in the hardware.}
    \label{fig:air_traj}
\end{figure}

However, classical OSCTC is essentially based on feedback linearization \cite{isidori1985nonlinear}. As such, it requires accurate knowledge of the system dynamics to achieve precise cancellation of nonlinear terms. When model uncertainties or disturbances are present, this assumption is violated, resulting in performance degradation and reduced robustness. This phenomenon is partially illustrated in Fig.~\ref{fig:air_traj},
where conventional OSCTC exhibits degraded tracking performance
in the presence of uncertainties. These uncertainties typically arise from both \emph{internal} and \emph{external} disturbances. \emph{Internal disturbances} are caused by imperfect dynamic modeling, payload variations, and unmodeled friction, while \emph{external disturbances} include environmental interactions such as unpredictable human-applied forces. Therefore, integrating operational space control with provably safe mechanisms under such disturbances remains a fundamental challenge.

%Redundant robotic manipulators operating in constrained and human-interactive environments must achieve accurate task-space tracking while guaranteeing safety under model uncertainty. Integrating operational-space control with formally certified safety in the presence of disturbances remains a fundamental challenge.

%\noindent

To enhance the robustness of OSCTC, two main approaches are commonly adopted. Data-driven or learning-based techniques \cite{wong2022oscar,o2022neural} approximate unknown dynamics but require full-state measurements, lack formal guarantees, and are sensitive to distribution shifts. Adaptive control schemes \cite{Seraji1989,tee2011adaptive,asar2019anfis,Feng1993} estimate structured uncertainties and are thus less effective against unstructured disturbances.

% To enhance the robustness of OSCTC against modeling errors and disturbances, two main approaches are commonly adopted. One direction integrates data-driven or learning-based techniques to approximate unknown dynamics \cite{wong2022oscar,o2022neural}; however, these methods require full-state measurements, lack formal guarantees, and can be sensitive to distribution shifts. Note that in Fig.~\ref{fig:air_traj}, the neural network requires task-specific retraining to achieve better performance and exhibits limited generalization across different trajectory-tracking scenarios in our experiments. The other direction incorporates adaptive control schemes \cite{Seraji1989,tee2011adaptive,asar2019anfis,Feng1993} to estimate uncertain parameters, which are typically assumed to be structured uncertainties and may be less effective against unstructured disturbances.

%Operational space control (OSC) \cite{Khatib1987OSC} provides a dynamically consistent framework for regulating end-effector motion while preserving null-space redundancy. When accurate dynamics are available, OSC achieves decoupled task-space tracking. However, practical manipulators are affected by friction, actuator dynamics, and modeling errors, particularly in closed-architecture systems~\cite{ZhangClosedArchitecture}, leading to disturbance-induced tracking and safety violations.

\begin{figure}[htbp]
    \centering
    \includegraphics[width=1\linewidth]{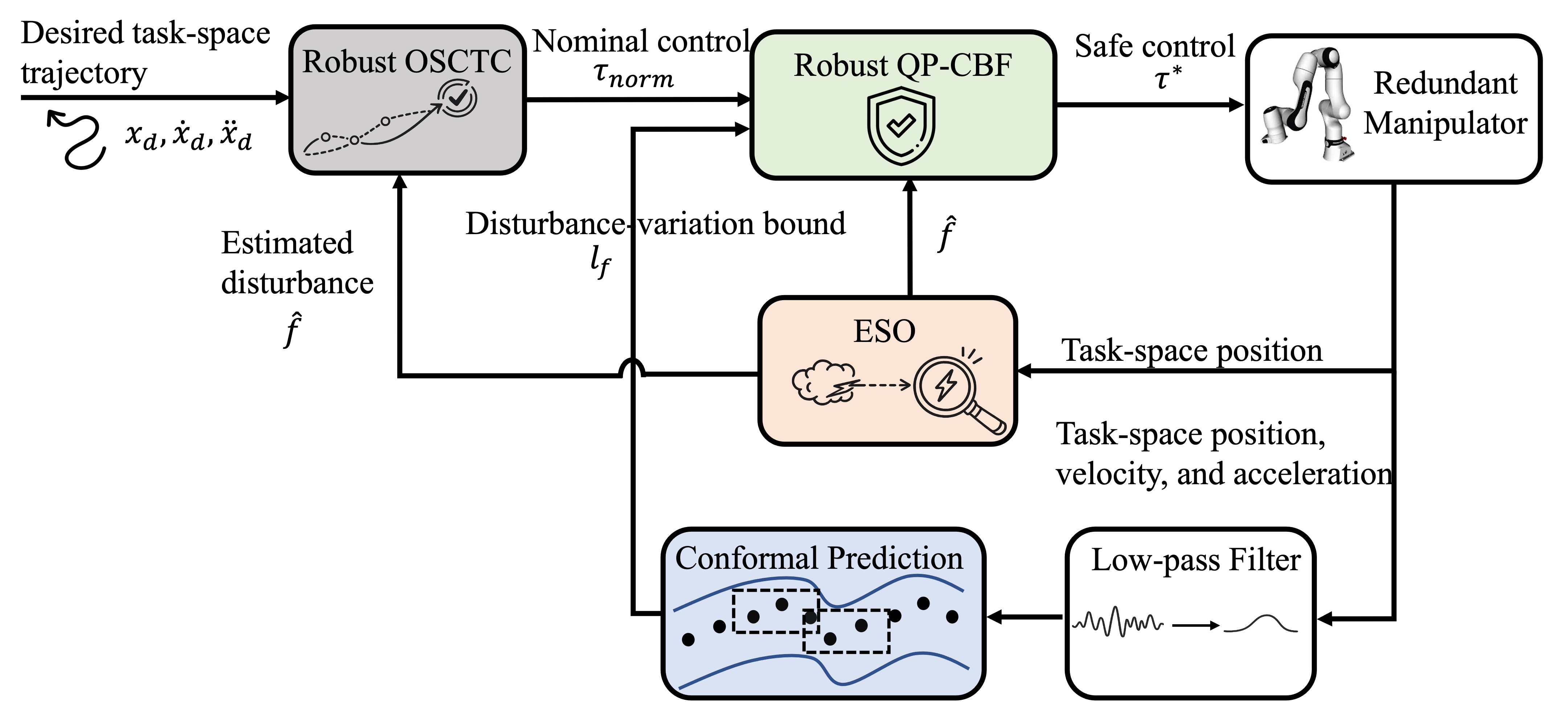}
    \caption{Architecture of the proposed robust operational-space control framework integrating model-based robustness and data-driven uncertainty quantification.}
    \label{fig:overview}
\end{figure}

In this work, we propose a robust OSCTC framework that integrates model-based robustness with data-driven uncertainty quantification. As illustrated in Fig.~\ref{fig:overview}, an extended state observer (ESO) \cite{han2009ADRC} (orange block) is designed to estimate lumped disturbances directly in operational space, enabling compensation of both structured and unstructured disturbances in the robust OSCTC (gray block). Based on the disturbance estimate, a robust control barrier function (CBF) (green block) is constructed to guarantee safety, \emph{e.g.}, collision avoidance, under disturbance. To address the limitation in ESO-based CBF designs for a known disturbance-variation bound, we further develop a sliding-window conformal prediction mechanism \cite{shafer2008tutorial} (blue block) to estimate this bound online in a data-driven and less conservative manner.

The main contributions of this work are summarized as follows.
\begin{enumerate}
    \item We provide a mechanistic analysis of residual learning and identify its inherent limitations. In contrast, the proposed framework (i) requires only position measurements rather than full-state information, whose derivative terms are often noisy in real robotic systems; (ii) guarantees state convergence with theoretical support; and (iii) maintains a simple control structure with minimal parameter tuning.
    \item We incorporate a key advantage of data-driven methods through conformal prediction, which enables online estimation of the disturbance variation bound. This eliminates the need for a conservatively pre-specified fixed bound in robust CBF design and significantly reduces over-conservatism.
    \item The proposed framework is validated on a 7-DoF Franka Research 3 (FR3) redundant manipulator. Experimental results demonstrate millimeter-level tracking accuracy and safety under various disturbances, while maintaining a real-time control frequency of 1 kHz.

\end{enumerate}

The work most closely related to ours is \cite{morton2025safe}, which integrates tracking control and CBFs in task space. However, disturbances are not considered in their work. In contrast, our approach explicitly accounts for dynamic uncertainties, providing a principled foundation for safe task-space control of redundant manipulators operating in dynamic environments.

The rest of the paper is organized as follows: Section II introduces the preliminaries and problem formulation. Section III presents the proposed
framework. Simulation and hardware experiments are reported
in Section IV. Concluding remarks
are in Section V.

% \section{Related Work}
% In this section, we will briefly review the most important
% related work on control of manipulators.

% \subsection{Learning-based control}

% \subsection{Robust operational space control}

% \subsection{Safe control}

\section{Preliminary and Problem Formulation}

This section reviews the dynamics of redundant manipulators, operational space control, CBF, and introduces the problem formulation. %Two compensation strategies are introduced: observer-based estimation and residual learning-based modeling.

%====================================================
\subsection{System Dynamics}

A general robotic system with an Euler-Lagrange formulation in joint space is described as

\begin{equation}
M(q)\ddot{q} + C(q,\dot{q})\dot{q} + G(q) = \tau,
\label{eq:joint_dyn_nominal}
\end{equation}
where 
$q, \dot{q}, \ddot{q} \in \mathbb{R}^{n}$ denote the joint position, velocity, and acceleration vectors, respectively. 
$M(q) \in \mathbb{R}^{n\times n}$ is the symmetric positive-definite inertia matrix. 
$C(q,\dot{q})\dot{q} \in \mathbb{R}^{n}$ represents the Coriolis and centrifugal forces. 
$G(q) \in \mathbb{R}^{n}$ is the gravity vector. 
$\tau \in \mathbb{R}^{n}$ denotes the control torque input. In our experiments, we set $n=7$, corresponding to the 7-DoF FR3 redundant manipulator used in our study.

Let $x \in \mathbb{R}^{m}$ denote the task-space variable (\emph{e.g.}, end-effector position or pose). The task-space kinematics are given by
\begin{align}
\dot{x} &= J(q)\dot{q}, \\
\ddot{x} &= J(q)\ddot{q} + \dot{J}(q,\dot{q})\dot{q},
\end{align}
where $J(q) \in \mathbb{R}^{m \times n}$ is the Jacobian matrix, and 
$\dot{J}(q,\dot{q})$ denotes its time derivative.

Substituting the joint-space acceleration 
\[
\ddot{q} = M^{-1}(q)\big(\tau - C(q,\dot{q})\dot{q} - G(q)\big)
\]
into the task-space acceleration yields
\begin{equation}
\ddot{x} 
= J(q) M^{-1}(q)\big(\tau - C(q,\dot{q})\dot{q} - G(q)\big)
+ \dot{J}(q,\dot{q})\dot{q}.
\label{eq:xddot_nominal}
\end{equation}

%====================================================
\subsection{Operational Space Computed Torque Control}
According to the OSCTC proposed in \cite{Khatib1987OSC}, the control torque is designed as
\begin{equation}
\tau = J^\top(q) F + N^\top(q)\tau_0,
\label{eq:OSCTC}
\end{equation}
where  $N(q) = I - M^{-1}(q) J^\top(q) \Lambda(q) J(q)$ is the dynamically consistent null-space projector. Here, $F \in \mathbb{R}^m$ is the task-space control force, and $\tau_0 \in \mathbb{R}^n$ is a secondary joint-space torque, \emph{e.g}., a posture Proportional-Derivative (PD) controller. In our experiments, we set $m=3$ to focus on task-space position tracking using $F$, while the posture controller $\tau_0$ operates in the null space to maintain the initial natural orientation.

Substituting \eqref{eq:OSCTC} into the task-space dynamics \eqref{eq:xddot_nominal} yields

\begin{equation}
\ddot{x}
= \Lambda^{-1}(q) F
+ \mu(q,\dot{q}),
\label{eq:xddot_task}
\end{equation}
where  $\Lambda(q) = \left(J(q) M^{-1}(q) J^\top(q)\right)^{-1}$ is the operational-space inertia matrix, and the nonlinear term is defined as
\begin{equation}
\mu(q,\dot{q})
=
- J(q) M^{-1}(q)\big(C(q,\dot{q})\dot{q} + G(q)\big)
+ \dot{J}(q,\dot{q})\dot{q}.
\end{equation}

To achieve exponentially stable tracking of a desired trajectory $x_d(t)$, the commanded task-space acceleration is designed as
\begin{equation}
a_{\mathrm{cmd}}
\triangleq
\ddot{x}_d
+ K_d(\dot{x}_d - \dot{x})
+ K_p(x_d - x),
\label{eq:acmd}
\end{equation}
where $K_p, K_d \in \mathbb{R}^{m \times m}$ are symmetric positive-definite gain matrices.

Enforcing $\ddot{x} = a_{\mathrm{cmd}}$ and using \eqref{eq:xddot_task}, the required task-space force is obtained as

\begin{equation}
F
=
\Lambda(q)
\big(
a_{\mathrm{cmd}}
-
\mu(q,\dot{q})
\big).
\label{eq:F_solution}
\end{equation}

% Under perfect modeling, substituting the above torque into \eqref{eq:xddot_nominal} yields

% \begin{equation}
% \ddot{x} = a_{\text{cmd}},
% \end{equation}

% resulting in  tracking error dynamics.

%====================================================
\subsection{Control Barrier Functions}

The task-space dynamics \eqref{eq:xddot_task} can be rewritten as a nonlinear control-affine system:
\begin{equation}\label{eq:affine}
    \dot{x}=\psi(x)+g(x)u,
\end{equation}
where $x \in \mathcal{X} \subset \mathbb{R}^{2m}$ includes task-space position and velocity, $\psi : \mathbb{R}^{2m} \rightarrow \mathbb{R}^{2m}$ and $g : \mathbb{R}^{2m} \rightarrow \mathbb{R}^{2m \times n}$ are Lipschitz continuous, and $u \in \mathcal{U} \subset \mathbb{R}^n$ is a control input vector.

The safety of system \eqref{eq:affine} can be guaranteed using a safety set $\mathcal{C}$ if it is forward invariant in the state space $\mathcal{X}$, \emph{i.e.}, for system \eqref{eq:affine} if solutions for some $u \in \mathcal{U}$ starting at any initial safe condition $x(0)  \in \mathcal{C}$ satisfy $x(t)  \in \mathcal{C}$, $\forall t \ge 0$.
The safety set $\mathcal{C}$ is defined as a 0-superlevel set of a continuously differentiable function \(h(x): \mathbb{R}^{2m} \rightarrow \mathbb{R}\) as:

\begin{equation}
    \mathcal{C} = \{x \in \mathbb{R}^{2m}~|~h(x) \ge 0\}.
\end{equation}

The function \(h\) is used to synthesize a controller with safety guarantees via a CBF. For a system with an arbitrary relative degree, we need the following higher-order CBF.
\begin{definition}
    (\emph{Exponential CBF (ECBF)} \cite{nguyen2016exponential, xiao2019control}) Consider system \eqref{eq:affine} with relative degree \(r\) for an \(r\)-times continuously differentiable function \(h\), \emph{i.e.}, \(L_gL_{\psi}h(x)=\cdots=L_gL_{\psi}^{r-2}h(x)=0\) and \(L_gL_{\psi}^{r-1}h(x)\neq 0, \forall x\in \mathcal{C}\). \(h(x)\) is an ECBF if there exists a row vector \(K_a\in\mathbb{R}^r\) satisfying \(\forall x\in\mathcal{C}\)
    
    \begin{equation}
        %\begin{array}{r@{}l}
            \sup\limits_{u \in \mathcal{U}} \big( L_{\psi}^r h(x)  +L_g L_{\psi}^{r-1}h(x)u \big) \geq -K_a \eta_b(x), 
        %\end{array}
    \end{equation}
\end{definition}
where \(\eta_b(x)=[h(x), \dot{h}(x), \cdots, {h^{(r-1)}(x)}]^T\), \(K_a=[k_1, \cdots, k_r]\), and the values of \(k_1, \cdots, k_r\) satisfy specific properties given in \cite{nguyen2016exponential, xiao2019control}. 

To leverage a CBF to guarantee safety, the control problem is formulated as a quadratic program (QP) with a CBF as a hard constraint. The QP formulation is as follows \cite{Ames2017CBF}:
\begin{equation}
\label{eq:QP-CBF}
\begin{split}
    u^* (x) & = \operatorname*{arg\,min}_{u\in U} \|u-k(x)\|^2 \\
    \text{s.t.} % \\
    \quad L_{\psi}^r h(x) & +L_g L_{\psi}^{r-1}h(x)u \geq -K_a \eta_b(x),
\end{split}
\end{equation}
where \(k(x)\) is a nominal control law.

\subsection{Problem Formulation}

In practice, exact robot dynamics are unavailable. 
We consider an uncertain manipulator with nominal model 
$\bar{M}(q)$, $\bar{C}(q,\dot{q})$, and $\bar{G}(q)$:
\begin{equation}
\bar{M}(q)\ddot{q} 
+ \bar{C}(q,\dot{q})\dot{q} 
+ \bar{G}(q) 
= \tau + d,
\label{eq:joint_uncertain}
\end{equation}
where $d \in \mathbb{R}^n$ denotes lumped joint-space disturbances, including modeling errors, friction, payload variations, and external interactions.

Mapping to task space yields

\begin{equation}
\ddot{x}
=
J \bar{M}^{-1}\tau
- J \bar{M}^{-1}\big(
\bar{C}\dot{q}
+ \bar{G}
\big)
+ \dot{J}\dot{q}
+ f,
\label{eq:task_uncertain}
\end{equation}
where the task-space disturbance is defined as

\begin{equation}
f \triangleq J \bar{M}^{-1} d.
\end{equation}

The objective is to achieve exponentially stable tracking of a desired task-space trajectory $x_d$ while ensuring safety in the presence of the unknown disturbance $f$.

%====================================================
% \subsection{Tracking Error Dynamics Under Disturbance}

% Define the tracking error

% \begin{equation}
% e = x - x_d.
% \end{equation}

% Substituting the nominal CTC torque into \eqref{eq:task_disturbance} gives

% \begin{equation}
% \ddot{e}
% + K_d \dot{e}
% + K_p e
% =
% f.
% \label{eq:error_with_disturbance}
% \end{equation}

% Thus, in the presence of disturbance $f$, the closed-loop behaves as a second-order system driven by an unknown input.

% Therefore, disturbance compensation is necessary to restore desired error convergence.

% %====================================================
% \subsection{E. Unified Disturbance Compensation Framework}

% We augment the nominal control by introducing a disturbance estimate $\hat{f}$:

% \begin{equation}
% F = F_{\text{nominal}} - \Lambda \hat{f}.
% \end{equation}

% Substituting into the dynamics yields

% \begin{equation}
% \ddot{e}
% + K_d \dot{e}
% + K_p e
% =
% f - \hat{f}.
% \label{eq:error_compensated}
% \end{equation}

% Thus, the residual disturbance

% \begin{equation}
% \tilde{f} = f - \hat{f}
% \end{equation}

% directly determines tracking performance.

% We next introduce two strategies to construct $\hat{f}$.

%====================================================
\section{Proposed Framework}

\subsection{Operational Space ESO-Based Disturbance Estimation}

For the $i$-th dimension of the task space, the uncertain dynamics in \eqref{eq:task_uncertain} can be written as a second-order system:

\begin{equation}
\begin{cases}
\dot{x}_{1i} = x_{2i}, \\
\dot{x}_{2i} = g_i(x,\tau) + f_i,
\end{cases}
\label{eq:2nd-before-ESO}
\end{equation}
where $x_{1i}$ and $x_{2i}$ are the task-space position and velocity at $i$-th dimension, respectively. 
$g = J \bar{M}^{-1}\tau 
- J \bar{M}^{-1}\big( \bar{C}\dot{q} + \bar{G} \big)
+ \dot{J}\dot{q}$ denotes the nominal model component. 
The scalars $g_i$ and $f_i$ represent the $i$-th components of $g$ and $f$, respectively.

Define the augmented state $x_{3i} = f_i$. 
An ESO is designed as
\begin{equation}
\begin{cases}
\dot{\hat{x}}_{1i} = \hat{x}_{2i} - L_{1i}(\hat{x}_{1i} - x_{1i}), \\
\dot{\hat{x}}_{2i} = \hat{x}_{3i} + g_i - L_{2i}(\hat{x}_{1i} - x_{1i}), \\
\dot{\hat{x}}_{3i} = - L_{3i}(\hat{x}_{1i} - x_{1i}),
\end{cases}
\label{eq:ESO}
\end{equation}
where $L_{1i}, L_{2i}, L_{3i} > 0$ are observer gains. Let the estimation error be defined as
\[
\vartheta_i =
\begin{bmatrix}
x_{1i} - \hat{x}_{1i} \\
x_{2i} - \hat{x}_{2i} \\
x_{3i} - \hat{x}_{3i}
\end{bmatrix}.
\]
Then, combining \eqref{eq:2nd-before-ESO} and \eqref{eq:ESO}, the error dynamics is
\begin{equation}
\dot{\vartheta}_i
=
A_i \vartheta_i
+
E_i \dot{f}_i,
\label{eq:error_dynamic}
\end{equation}
where
\[
A_i =
\begin{bmatrix}
- L_{1i} & 1 & 0 \\
- L_{2i} & 0 & 1 \\
- L_{3i} & 0 & 0
\end{bmatrix},
\quad
E_i =
\begin{bmatrix}
0 \\
0 \\
1
\end{bmatrix}.
\]

The observer gains are selected such that $A_i$ is Hurwitz. 
In practice, the poles are placed at $-\omega_{o_i}$, where $\omega_{o_i} > 0$ denotes the observer bandwidth, serving as the only tuning parameter of the ESO \cite{gao2003scaling}.

\subsection{Residual Learning vs. Dynamic Disturbance Estimation}

A widely used paradigm in the learning-based control community models the unknown term $f_i$ as a residual correction to the nominal dynamics. 
Rearranging \eqref{eq:2nd-before-ESO} yields
\begin{equation}
f_i = \dot{x}_{2i} - g_i(x,\tau),
\label{eq:residual_algebraic}
\end{equation}
which requires access to the acceleration $\dot{x}_{2i} = \ddot{x}_i$. 
The resulting residual estimate is then used as a supervised learning target to train a neural network (NN),
\begin{equation}
\hat{f}_i = \phi_i(z),
\end{equation}
where $z$ collects measurable states such as 
$z = [x^\top,\dot{x}^\top,\tau^\top]^\top$.

\begin{rem}[Residual Learning vs. ESO]
Compared with residual learning, the proposed operational-space ESO (i) needs only position measurements, avoiding the noisy numerical differentiation that residual learning requires for velocity and acceleration; (ii) embeds disturbance estimation in a Hurwitz observer with explicit exponential error convergence, rather than relying on convergence of a learned model; and (iii) adds a single bandwidth parameter instead of network architecture and training design.
\end{rem}

\subsection{Robust OSCTC design}
After estimating the disturbance $\hat{f}$, the controller in \eqref{eq:F_solution} is reformulated in a robust form to compensate for its effect:
\begin{equation}
    F = \bar{\Lambda} a_{\text{cmd}} + \bar{\Lambda}(J\bar{M}^{-1}\bar{C} - \dot{J}\dot{q}) + {\bar{\Lambda} J \bar{M}^{-1}\bar{G}} - \bar{\Lambda} \hat{f}.
\end{equation}

We will show below that, even in the presence of disturbances, the null-space torque does not affect the task-space tracking guarantee.

\begin{lemma}[Null-space decoupling under uncertain inertia]
\label{lem:null_uncertain}
Let $\bar M=\bar M^\top\succ 0$ and assume $J$ has full row rank so that
$\bar\Lambda \triangleq (J\bar M^{-1}J^\top)^{-1}$ exists. Define the
$\bar M$-weighted generalized inverse and the associated projector
\begin{equation}
\tilde J^\# \triangleq \bar M^{-1}J^\top\bar\Lambda,\qquad
\bar N \triangleq I-\tilde J^\#J.
\end{equation}
Then the null-space torque $\tau_0=\bar N^\top \eta$ does not affect the task
space, \emph{i.e.},
\begin{equation}
J\bar M^{-1}\bar N^\top=0.
\end{equation}
\end{lemma}

\begin{proof}
Since $\bar M$ and $\bar\Lambda$ are symmetric, $(\tilde J^\#)^\top=\bar\Lambda
J\bar M^{-1}$. Hence
\[
J\bar M^{-1}\bar N^\top
=J\bar M^{-1}\!\left(I-J^\top(\tilde J^\#)^\top\right)
=\left(I-(J\bar M^{-1}J^\top)\bar\Lambda\right)J\bar M^{-1}=0,
\]
where $(J\bar M^{-1}J^\top)\bar\Lambda=I$ by definition.
\end{proof}

\subsection{Robust High-Order CBF for Safe Control}

\begin{assumption}\label{asm1}
There exists a positive known constant \(l_{fi}\) such that for any \(x\in \mathcal{X}\), \(u\in \mathcal{U}\), and \(t\geq0\), the following inequality holds:
    \begin{equation}\label{eq_6}
        \left|\dfrac{\partial f_i(t,x,u)}{\partial t}\right|\leq l_{fi},
    \end{equation}
\end{assumption}
where $i$ indicates the $i$-th  dimension of the dynamic system, see \eqref{eq:2nd-before-ESO}. Assumption \ref{asm1} implies that \(f_i\) is locally Lipschitz continuous with respect to time, and $\dot{f}_i$ is bounded. 

The system  \eqref{eq:task_uncertain} with disturbance is formulated into a control affine form:
\begin{equation}\label{eq:RCBF_dynamics}
    \dot{x}=\psi(x)+g(x)u+\bar{g}f,
\end{equation}
where \(f=[f_1, \cdots, f_m]^T \in \mathcal{F} \subset \mathbb{R}^m\) is a total disturbance vector in each input channel, {\(\bar{g}=[\mathbf{0}_{m\times m}, \mathbf{I}_{m\times m}]^T\), \(\mathbf{0}_{m\times m}\) and \(\mathbf{I}_{m\times m}\)} are \(m\times m\) zero and identity matrices, respectively. \(f\) is a matched disturbance vector. To define a safety set $\mathcal{C}_s$ using a high-order CBF approach \cite{xiao2019control}, we consider an \(r\)-times continuously differentiable function \(h(x)\), where the relative degrees of \(h(x)\) with respect to both \(u\) and \(f\) are \(r\), given that \(f\) is matched. A sequence of functions is defined recursively: $h_0(x)=h(x)$ and $h_i(x)=\dot{h}_{i-1}(x)+\gamma_i h_{i-1}(x)$ for $i=1,\dots,r$. The safe set is $\mathcal{C}_s = \cap_{i=0}^r \{x \mid h_i(x) \geq 0\}$.

Based on \cite{chen2025model}, the disturbance estimation error is bounded in the discrete domain:
% The error $\tilde{f} = f - \hat{f}$ is bounded by a vector $\Gamma \in \mathbb{R}^n_{+}$:
\begin{equation}\label{eq:RCBF_ErrorBound}
        |f_i(k)-\hat{f}_i(k)|\leq\Gamma_i(\omega_{o_i}, l_{fi}, T_s),
\end{equation}
where the detailed expression of $\Gamma$ can be found in \cite{chen2025model}. For simplicity, this error bound depends only on the sampling time $T_s$, the discrete-time observer gain $\omega_{o_i}$, and the disturbance-variation bound $l_{fi}$.
\begin{theorem}\label{Th_RCBFwithFhat}
    Given the system \eqref{eq:RCBF_dynamics} and the ESO in \eqref{eq:ESO} under Assumption \ref{asm1}, any controller \(u(x)\in K_{\text{rcbf}}\) renders the set \(\mathcal{C}_s\) forward invariant for system \eqref{eq:RCBF_dynamics}, where
    \begin{equation}\label{eq:RCBFUSetFbound}
        \begin{array}{r@{}l}
            K_{\text{rcbf}}(t,x,f,\hat{f}) \triangleq & \{u\in \mathcal{U}: L_{\psi}^r h(x) + L_g L_{\psi}^{r-1} h(x)u \\
             & + L_{\bar{g}} L_{\psi}^{r-1} h(x)\hat{f} - |L_{\bar{g}} L_{\psi}^{r-1} h(x)|\Gamma(\omega_o, l_f, T_s) \\
             & +\sum\limits_{j=0}^{r-1}k_j h^{(j)}(x) \geq 0\}, 
        \end{array}
    \end{equation}
    where \(\Gamma(\omega_o, l_f, T_s)=[\Gamma_1(\omega_{o_1}, l_{f1}, T_s), \cdots, \Gamma_m(\omega_{o_m}, l_{fm}, T_s)]^T\). 
\end{theorem}

% Due to space limitations, the detailed proof is omitted. 
The main challenge is that the disturbance-variation bound, denoted by $l_{fi}$ in Assumption~\ref{asm1}, is typically unknown in practice. 
In the following section, we employ conformal prediction as a data-driven approach to estimate this bound.

We consider wall-like safety constraints defined along a Cartesian axis.
Define the scalar projection
\begin{equation}
\zeta := e^\top x(q),
\qquad
\dot{\zeta} = e^\top \dot{x} = e^\top J(q)\dot{q},
\label{eq:zeta_def}
\end{equation}
where $e \in \mathbb{R}^3$ is a unit selection vector, \emph{i.e.},
\begin{equation}
\label{eq:e_definition}
e^\top \in \{[1,0,0],\ [0,1,0],\ [0,0,1]\},
\end{equation}
corresponding to the $x$-, $y$-, and $z$-axis directions, respectively.

To enforce a scalar constraint of the form
\begin{equation}
\label{eq:scalar_constraint_definition}
\sigma (\zeta - \zeta_b) \le 0,
\qquad
\sigma \in \{-1,1\},
\end{equation}
where $\zeta_b$ denotes the boundary value,
$\sigma=-1$ corresponds to the upper bound $\zeta \le \zeta_b$,
and $\sigma=+1$ corresponds to the lower bound $\zeta \ge \zeta_b$. Define the barrier function
\begin{equation}
h(\zeta) = -\sigma (\zeta - \zeta_b).
\end{equation}

Since the relative degree of system \eqref{eq:task_uncertain} is two, the ECBF condition is
\begin{equation}
\ddot{h} + k_1 \dot{h} + k_0 h \ge 0,
\qquad k_0,k_1>0.
\end{equation}

Our final QP-CBF is:

\begin{equation}
\label{eq:QP-CBF}
\begin{aligned}
\tau^* 
&= \operatorname*{arg\,min}_{\tau \in \mathcal{T}}
    \;\|\tau - k(x)\|^2 \\
\text{s.t.}\quad 
& a^\top \tau \le \beta .
\end{aligned}
\end{equation}
where \(k(x)\) is a nominal control law, $a^\top = (\sigma e)^\top J \bar{M}^{-1}$, $b = (\sigma\,e)^{\top}\big(\hat{f}+\Gamma+\dot{J}\dot{q} - J \bar{M}^{-1}(\bar{C}+\bar{G})\big)$, $\beta = -k_1 \dot{h} - k_0 h - b$.

\subsection{Probabilistic Robust CBF via Sliding-Window Conformal Disturbance Bounds}

% \subsection{Role of Disturbance Derivative in ESO Error Dynamics}

% Consider the task-space dynamics

% \begin{equation}
% \dot{x}_1 = x_2, \quad
% \dot{x}_2 = g(x,\tau) + f,
% \end{equation}

% where $f$ represents the lumped disturbance.

% Define the disturbance estimation error:

% \begin{equation}
% \tilde{f} = f - \hat{f}.
% \end{equation}

% For the linear ESO structure, the estimation error dynamics satisfy

% \begin{equation}
% \dot{\tilde{f}} = -L \tilde{f} + \dot{f}.
% \end{equation}

% Solving the linear differential equation yields

% \begin{equation}
% \tilde{f}(t)
% =
% \int_0^t e^{-L(t-s)} \dot{f}(s) ds.
% \end{equation}

% Hence,

% \begin{equation}
% |\tilde{f}(t)|
% \le
% \int_0^t e^{-L(t-s)} |\dot{f}(s)| ds.
% \end{equation}

% Therefore, deterministic robustness requires

% \begin{equation}
% |\dot{f}(t)| \le l_f,
% \end{equation}

% which implies

% \begin{equation}
% |\tilde{f}(t)| \le \gamma(l_f),
% \end{equation}

% where $\gamma(l_f)$ depends on observer bandwidth.

% This derivation reveals that bounded disturbance derivative is a structural requirement for deterministic ESO-based robust CBF formulations.

%\subsection{Sliding-Window Conformal Disturbance Derivative Bound}

To relax the deterministic assumption on the disturbance-variation bound, we estimate a distribution-free bound on the disturbance derivative using sliding-window conformal prediction online. The disturbance of the $i$-th dimension of the system \eqref{eq:2nd-before-ESO} is computed from the algebraic residual
\begin{equation}
f_i^{\mathrm{res}}(k) = \ddot{x}_i(k) - g_i(x(k),\tau(k)),
\end{equation}
and its discrete-time derivative is approximated by
\begin{equation}
r(k) = \dot{f}_i^{\mathrm{res}}(k)
= \frac{f_i^{\mathrm{res}}(k) - f_{i}^{\mathrm{res}}(k-1)}{T_s}.
\end{equation}

At time step $k$, consider a sliding window of size $N$:
\begin{equation}
\mathcal{R}(k)
=
\{ |r(k-N+1)|, \dots, |r(k)| \}.
\end{equation}

The empirical $(1-\alpha)$ quantile is defined as
\begin{equation}
l_i^\alpha(k)
=
\text{Quantile}_{1-\alpha}(\mathcal{R}(k)).
\end{equation}

Under the exchangeability assumption within the window \cite{papadopoulos2002inductive}, we obtain
\begin{equation}
\mathbb{P}
\big(
|\dot f_i^{\mathrm{res}}(k)|
\le
l_i^\alpha(k)
\big)
\ge
1-\alpha.
\end{equation}

\begin{rem}
In this work, we estimate the disturbance-variation bound directly from the residual rather than from the ESO-estimated disturbance. Although the residual signal is more noisy, it more faithfully reflects the variation of the true disturbance. In contrast, the disturbance estimated by the ESO is significantly smoothed due to the observer bandwidth, which attenuates high-frequency variations and may lead to an underestimation of the disturbance variation.
\end{rem}

%\subsection{Connecting the conformal bound to the ESO error bound}

As shown in~\eqref{eq:RCBF_ErrorBound}, the ESO estimation error is bounded by 
$\Gamma_i(\omega_{o_i}, l_{fi}, T_s)$, 
where the key quantity to be identified is $l_{fi}$, denoting an (unknown) upper bound on the disturbance variation. We set the disturbance-variation bound used in $\Gamma_i(\cdot)$ as
\begin{equation}\label{eq:lf_conformal}
l_{fi}(k) \triangleq l_i^\alpha(k),
\end{equation}
which yields an \emph{online, probabilistic} ESO error bound
\begin{equation}\label{eq:Gamma_online}
|f_i(k)-\hat{f}_i(k)|
\le
\Gamma_i\!\big(\omega_{o_i},\, l_{fi}(k),\, T_s\big)
=
\Gamma_i\!\big(\omega_{o_i},\, l_i^\alpha(k),\, T_s\big).
\end{equation}
The resulting bound inherits the coverage level $1-\alpha$ from the conformal prediction. After the error bound is determined, the robust CBF \eqref{eq:RCBFUSetFbound}, together with the estimated disturbance, is enforced online to guarantee safety.

\begin{rem}[Deterministic vs. Probabilistic Safety Guarantee]

If the disturbance variation bound $l_f$ is known a priori and satisfies
$|\dot f(k)| \le l_f$ for all $k$, 
then the ESO error bound in~\eqref{eq:RCBF_ErrorBound} holds deterministically.
Consequently, the resulting CBF condition \eqref{eq:RCBFUSetFbound} ensures forward invariance of the safe set $\{x : h(x) > 0\}$ with absolute guarantee.

In contrast, when $l_f$ is unknown and replaced by the conformal estimate
$l^\alpha(k)$, the ESO error bound becomes probabilistic.
Accordingly, the forward invariance of the safe set is guaranteed with probability at least $1-\alpha$,
under the exchangeability assumption of the sliding window.
\end{rem}

\section{Experiments and Results}

This section evaluates the proposed framework and the comparison with residual learning approach in both simulation and hardware experiments. 
% The objectives are:

% \begin{enumerate}
%     \item Validate disturbance estimation accuracy,
%     \item Quantify tracking performance under model mismatch,
%     \item Demonstrate safety guarantees via Control Barrier Functions (CBFs),
%     \item Analyze the contribution of each module through ablation studies.
% \end{enumerate}

%%%%%%%%%%%%%%%%%%%%%%%%%%%%%%%%%%%%%%%%%%%%%%%%%%%%%%%%%%%%%%%%%%%%%%%%%%
\subsection{Experimental Platform and Implementation Details}

All experiments are performed on a 7-DoF FR3 manipulator. Both simulation and hardware implementations operate at 1~kHz. 
The control framework is implemented in C++ within the Franka ROS stack, and the CBF-QP is solved in real time using qpOASES. The residual neural network is trained offline and deployed online via ONNX Runtime for efficient real-time inference.

%To mitigate undesirable transient excitation and observer initialization effects, a smooth minimum-jerk ramp-in phase is applied during the first few seconds of each experiment.During this period, the trajectory amplitude is gradually increased from zero, allowing the extended state observer (ESO) and real-time uncertainty estimator to converge before full disturbance compensation and safety filtering are activated.
% The architecture consists of:

% \begin{itemize}
%     \item Operational Space Controller (OSC),
%     \item Extended State Observer (ESO),
%     \item Residual neural network disturbance model,
%     \item CBF-based quadratic program (QP) safety filter.
% \end{itemize}

%%%%%%%%%%%%%%%%%%%%%%%%%%%%%%%%%%%%%%%%%%%%%%%%%%%%%%%%%%%%%%%%%%%%%%%%%%
\subsubsection{Disturbance Injection Model}

To evaluate robustness, we inject nonlinear state-dependent disturbances in task space:
\begin{equation}
\label{eq:added_disturbance}
d_i(x,\dot{x}) = 
a_{p,i} \tanh\!\left(\frac{x_i - x_{0,i}}{w}\right)
+ a_{v,i} \dot{x}_i |\dot{x}_i|.
\end{equation}

The first term introduces bounded nonlinear position-dependent disturbance, 
while the second term models quadratic velocity-dependent damping. This disturbance is smooth and Lipschitz continuous, satisfying Assumption \ref{asm1}.

%%%%%%%%%%%%%%%%%%%%%%%%%%%%%%%%%%%%%%%%%%%%%%%%%%%%%%%%%%%%%%%%%%%%%%%%%%

\subsubsection{Safety setup}
In our experiments, we choose $e^\top = [0,1,0]$ in \eqref{eq:e_definition}, corresponding to a constraint along the $y$-axis. We set $\sigma = -1$ and $\zeta_b = 0.3\,\text{m}$ in \eqref{eq:scalar_constraint_definition}, so that the safe set is defined as $\{y \le 0.3\,\text{m}\}$.

% To enforce the upper bound $y \le y_{\max}$, define the barrier function \[ h(y) = y_{\max} - y = sign \,(y - y_{\max}), \quad sign = -1. \] Thus, \[ \dot{h} = sign \,\dot{y}, \qquad \ddot{h} = sign \,\ddot{y}. \] Since the relative degree is 2, the high-order exponential CBF (ECBF) condition is \[ \boxed{\ddot{h} + k_1 \dot{h} + k_0 h \ge 0}, \quad k_0, k_1 > 0. \] Substitute $\ddot{h} = s\,\ddot{y}$ and $\dot{h} = s\,\dot{y}$: \[ s\,e_y^{\top}\big( J \bar{M}^{-1} \tau + f - J \bar{M}^{-1} \bar{C}^{-1} - J \bar{M}^{-1} \bar{G}^{-1} + \dot{J} \dot{q}\big) + k_1 (s\,\dot{y}) + k_0 (y_{\max} - y) \ge 0. \] Collect the $\tau$ terms on the left-hand side: \[ (s\,e_y)^{\top}J \bar{M}^{-1}\tau \ge -k_1 (s\,\dot{y}) - k_0 (y_{\max} - y) - (s\,e_y)^{\top}(f - J \bar{M}^{-1} \bar{C}^{-1} - J \bar{M}^{-1} \bar{G}^{-1} + \dot{J} \dot{q}). \] 

% \[ { \begin{aligned} a^{\top} &= (\sigma\,e)^{\top}J \bar{M}^{-1},\\ b &= (\sigma\,e)^{\top}\big(\hat{f}+\Gamma+\dot{J}\dot{q} - J \bar{M}^{-1}(\bar{C}+\bar{G})\big),\\ \beta &= -k_1 \dot{h} - k_0 h - b, \end{aligned}} \] so the constraint becomes \[ {a^{\top}\tau \ge \beta.} \]

\begin{figure}[htbp]
    \centering
    \includegraphics[width=0.9\linewidth]{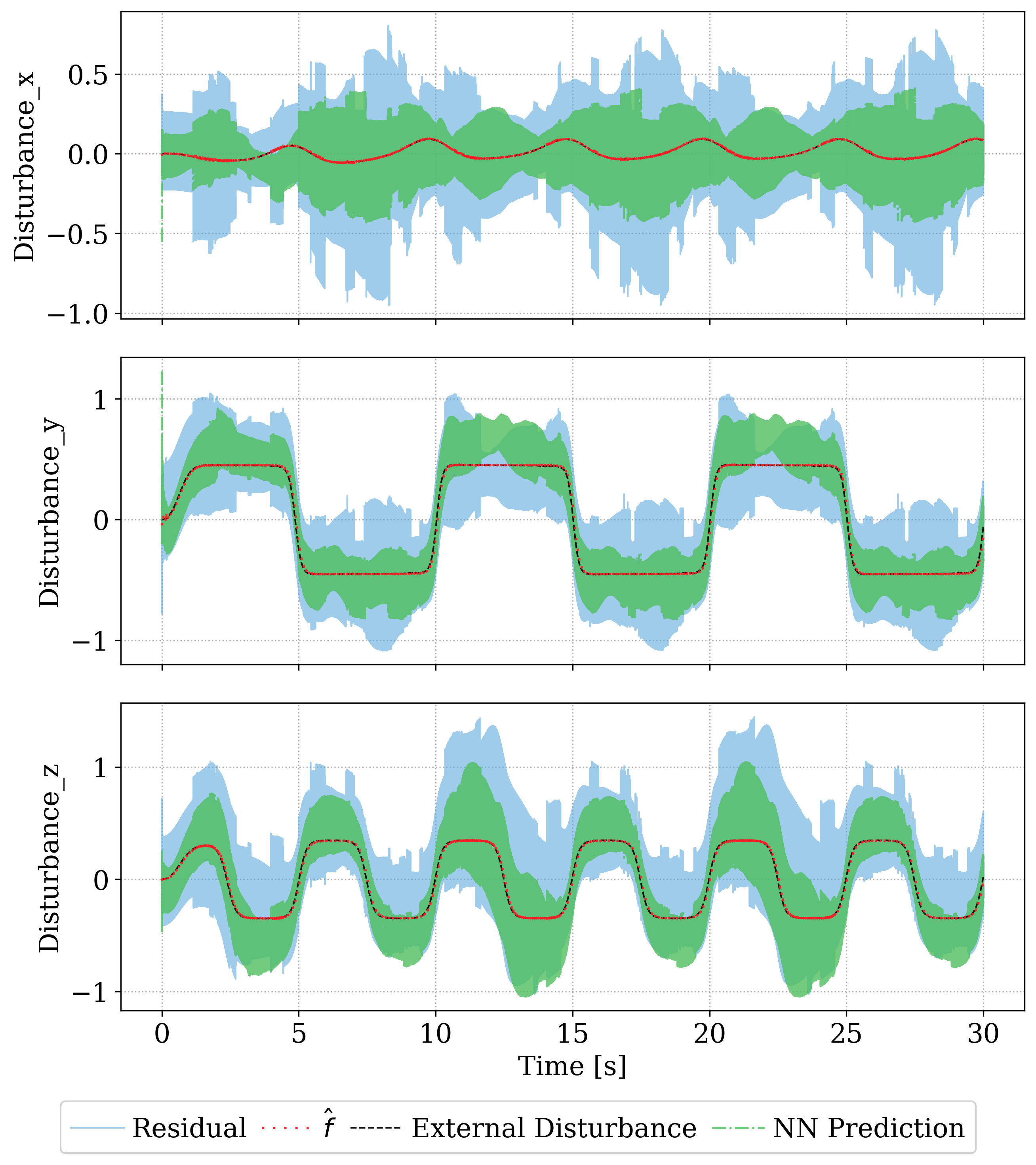}
    \caption{Disturbance estimation in simulation environment. 
    % Ground truth residual (black), ESO estimate (blue), and neural-network prediction (red).
    }
    \label{fig:sim_residual}
\end{figure}

\subsection{Simulation: Disturbance Estimation}

We first validate disturbance estimation in simulation. In FR3's simulation, the model errors in \( \bar{M}(q), \bar{C}(q,\dot{q}), \bar{G}(q) \) are negligible, and the dominant disturbance arises from the injected external signal in~\eqref{eq:added_disturbance}, allowing direct comparison between the ground-truth external disturbance and different estimation methods. We compare the computed residual, the ESO estimate \( \hat{f} \), and the NN-based prediction.

As shown in Fig.~\ref{fig:sim_residual}, the injected external disturbance (dashed black) exhibits the same overall trend as the computed residual (blue). The residual is noticeably noisier because it captures both internal (model mismatch) and external disturbances and is sensitive to state measurement noise. The NN prediction also follows the residual trend, validating its fitting accuracy. In contrast, the proposed ESO closely tracks the true external disturbance with reduced variation, as the observer bandwidth attenuates high-frequency components.

% the ESO estimate closely follows 
% the dominant trend of the true residual disturbance while exhibiting 
% clear smoothing characteristics. Compared to algebraic residual 
% reconstruction, the ESO avoids explicit numerical differentiation, 
% thus mitigating high-frequency noise amplification.

% Moreover, the ESO estimate aligns well with the injected external disturbance, 
% demonstrating effective disturbance tracking with limited phase lag. 
% This validates the noise-robust structure of dynamic disturbance estimation.

%%%%%%%%%%%%%%%%%%%%%%%%%%%%%%%%%%%%%%%%%%%%%%%%%%%%%%%%%%%%%%%%%%%%%%%%%%
\subsection{Hardware: Tracking Under Sudden External Disturbance}

% In the first hardware experiment, the end effector is commanded to track a lemniscate reference trajectory. A smooth ramp-in phase is applied during the first 5 seconds to mitigate the initial tracking transient, as the robot starts with a nonzero offset from the reference trajectory. An external disturbance is introduced at $t=15$~s by attaching a partially filled water bottle to the end-effector. This setup is nontrivial because the suddenly added payload changes the system inertia, and the \emph{sloshing} liquid generates time-varying disturbances.

% We compare the proposed method with conventional OSCTC without disturbance compensation \cite{Khatib1987OSC} and a residual learning-based method widely used in the learning-based control community. Fig.~\ref{fig:hardware_tracking_disturbance} shows the task-space trajectory and the tracking error. Before the disturbance is injected, all controllers exhibit comparable tracking performance. After $t=15$~s, conventional OSCTC (solid blue) shows noticeable deviation and increased tracking error. The residual learning method (dashed green) has slightly better performance than conventional OSCTC, but is much worse than ours (dashed red). Our controller shows superior robustness with stable tracking errors before and after the disturbance injection.

In the first hardware experiment, the end effector tracks a lemniscate trajectory, with a 5-second ramp-in to mitigate the initial transient from its nonzero starting offset. At $t=15$~s we attach a partially filled water bottle, whose added mass changes the inertia and whose \emph{sloshing} liquid generates time-varying disturbances.

\begin{figure}[h!]
    \centering
    \includegraphics[width=1\linewidth]{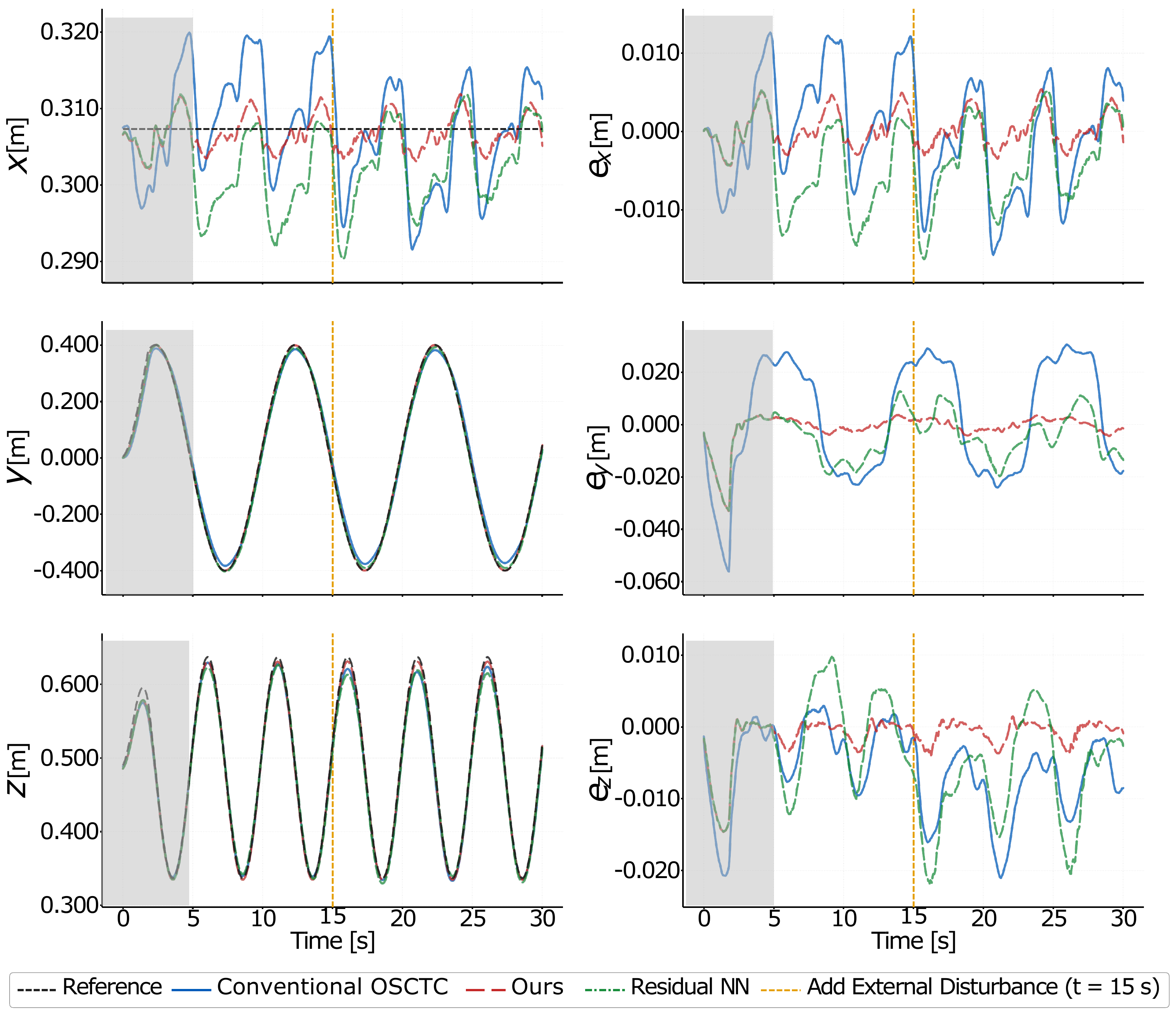}
    \caption{Hardware trajectory tracking under external disturbance 
    (left: trajectory tracking; right: tracking error; sloshing payload attached at $t=15$~s, 
    see the vertical dashed yellow lines). 
    The ramp-in phase in the first 5 seconds is colored in grey. 
    Overall mean-square error (MSE, unit [m]): Conventional OSCTC: $2.033\mathrm{e}{-04}$, Ours: $1.712\mathrm{e}{-05}$, Residual NN: $7.531\mathrm{e}{-05}$.}
    \label{fig:hardware_tracking_disturbance}
\end{figure}

We compare against conventional OSCTC without compensation \cite{Khatib1987OSC} and a widely used residual learning method (Fig.~\ref{fig:hardware_tracking_disturbance}). All controllers track comparably before injection; after $t=15$~s, conventional OSCTC (solid blue) deviates noticeably, residual learning (dashed green) does slightly better, and ours (dashed red) keeps a stable tracking error throughout.

\begin{figure}[h!]
    \centering
    \includegraphics[width=0.9\linewidth]{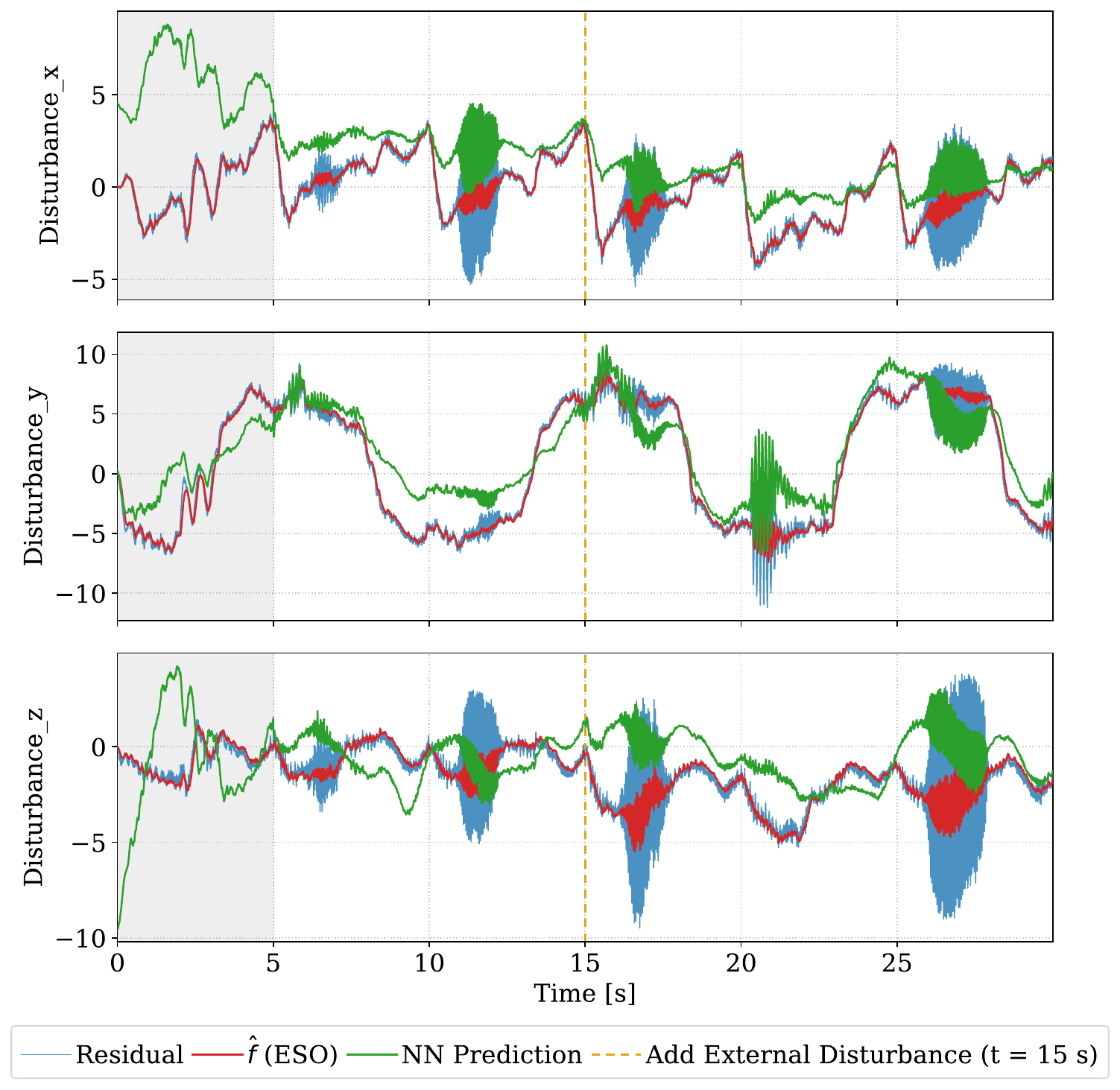}
    \caption{Hardware disturbance estimation under external disturbance (sloshing payload attached at t = 15 s, see the vertical dashed yellow lines). The ramp-in phase in the first 5 seconds is colored in grey.}
    \label{fig:hardware_disturbance}
\end{figure}

%%%%%%%%%%%%%%%%%%%%%%%%%%%%%%%%%%%%%%%%%%%%%%%%%%%%%%%%%%%%%%%%%%%%%%%%%%
\begin{figure*}[h!]
    \centering
    \includegraphics[width=0.9\textwidth]{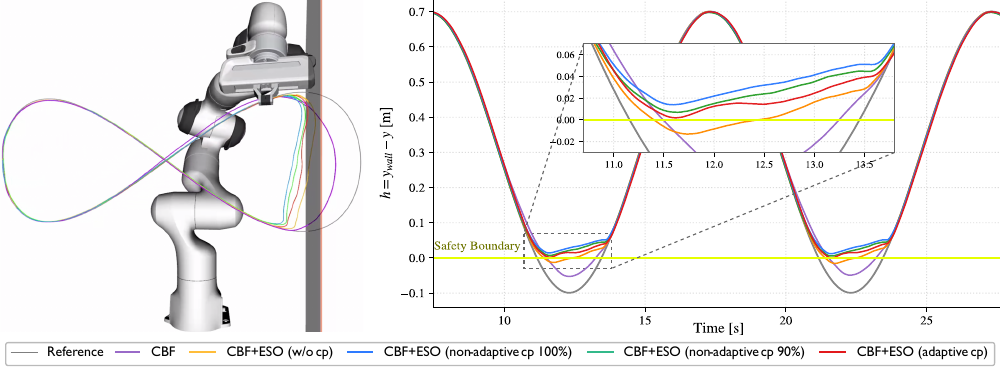}
    \caption{Hardware safety validation. 
    Left: RVIZ visualization of the physical robot and the safety boundary. 
    Right: Safety function $h(x)$.}
    \label{fig:cbf_hw}
\end{figure*}
% The disturbance estimation results are shown in Fig.~\ref{fig:hardware_disturbance}. The trend of the ESO estimate (red) aligns well with the residual (blue); however, the residual is significantly noisier. After the sloshing payload is attached, the variation of the residual increases substantially. The NN estimate (green) tracks the residual trend but with a noticeable bias.

Fig.~\ref{fig:hardware_disturbance} shows the disturbance estimation: the ESO estimate (red) follows the residual (blue) but is significantly less noisy, and the residual variation grows substantially once the payload is attached. The NN estimate (green) tracks the residual trend but with a noticeable bias.

%%%%%%%%%%%%%%%%%%%%%%%%%%%%%%%%%%%%%%%%%%%%%%%%%%%%%%%%%%%%%%%%%%%%%%%%%%
\subsection{Hardware: Safety Evaluation}

% We now evaluate safety performance with the safety specification $y \le 0.3$~m. In this experiment, the disturbance is mainly from the model uncertainty, since model errors in FR3 hardware are much larger than the simulation environment for $\bar{M}(q)$, $\bar{C}(q,\dot{q})$, and $\bar{G}(q)$. We have compared the following methods: (1) Conventional CBF without disturbance compensation; (2) Robust CBF with disturbance compensation using ESO; (3) Robust CBF with disturbance compensation using ESO and fixed disturbance estimation bound; (4) Robust CBF with disturbance compensation using ESO and adaptive disturbance estimation bound (ours).

% As shown in Fig.~\ref{fig:cbf_hw}, the reference trajectory is not strictly followed by any controller with CBF, since the QP-based CBF overrides the nominal control when safety constraints become active. The conventional CBF without disturbance compensation (purple) violates the safety constraint. Incorporating ESO without conformal prediction (yellow) improves safety performance, but violations still occur due to estimation errors. In contrast, the proposed CBF with ESO and adaptive conformal prediction (red) guarantees safety while being less conservative than the versions with fixed 90\% (green) and 100\% (blue) confidence levels.

We evaluate safety under the specification $y \le 0.3$~m, where the disturbance stems mainly from model uncertainty, which is far larger on FR3 hardware than in simulation. We compare the following methods: (1) Conventional CBF without disturbance compensation; (2) Robust CBF with disturbance compensation using ESO; (3) Robust CBF with disturbance compensation using ESO and fixed disturbance estimation bound; (4) Robust CBF with disturbance compensation using ESO and adaptive disturbance estimation bound (ours).
As shown in Fig.~\ref{fig:cbf_hw}, the reference trajectory is not strictly followed by any controller with CBF, since the QP overrides the nominal control whenever the constraint is active. The conventional CBF without disturbance compensation (purple) violates the safety constraint. Incorporating ESO without conformal prediction (yellow) improves safety performance, but violations still occur due to estimation errors. In contrast, the proposed CBF with ESO and adaptive conformal prediction (red) guarantees safety while being less conservative than the versions with fixed 90\% (green) and 100\% (blue) confidence levels.
\begin{figure}[htbp]
    \centering
    \includegraphics[width=1\linewidth]{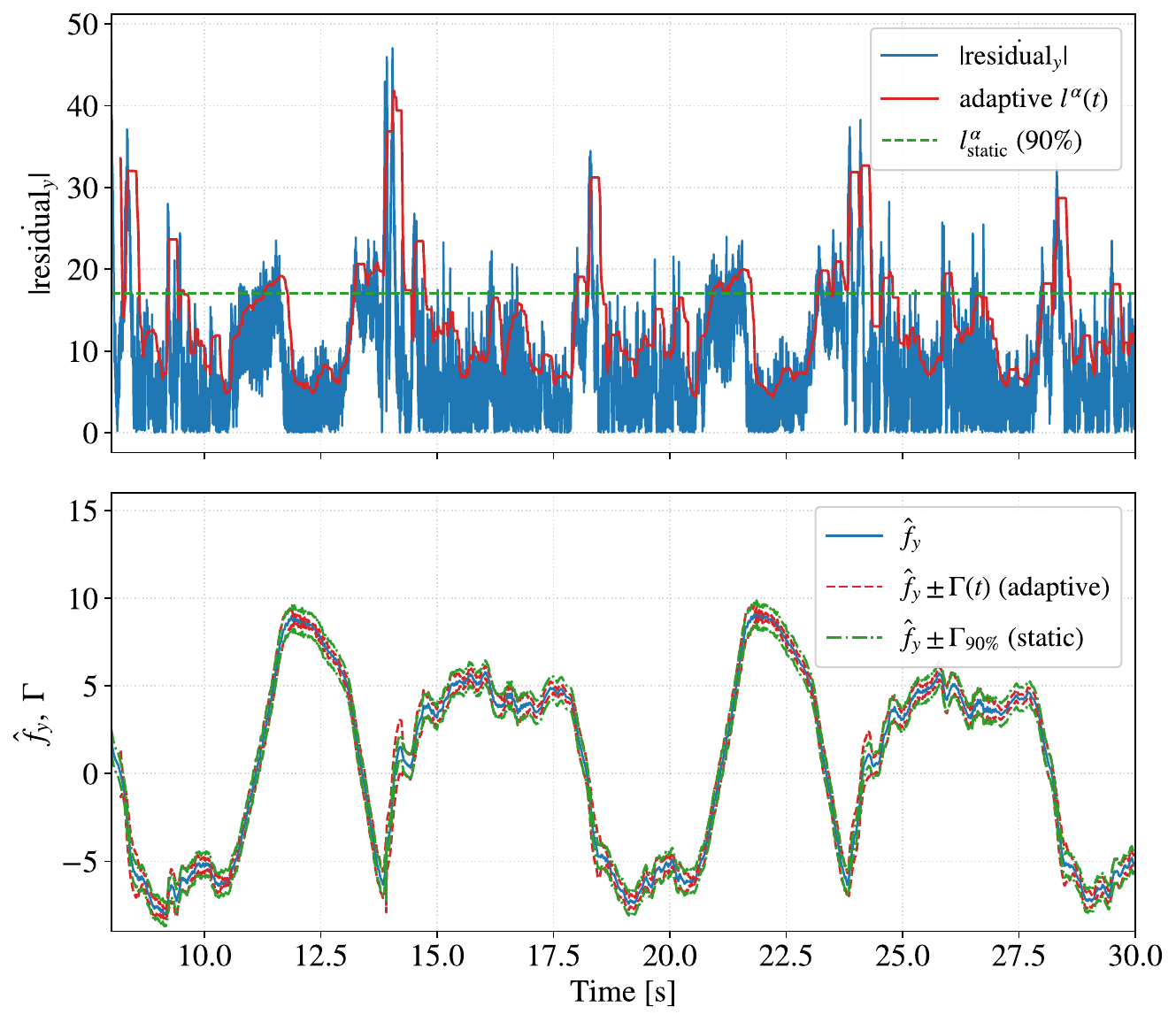}
    \caption{Adaptive conformal disturbance bound $l^\alpha(t)$ 
    and disturbance estimation error.}
    \label{fig:cp_bound}
\end{figure}
We plot the estimated disturbance and the error bounds along the $y$-axis as an example. As shown in the top of Fig.~\ref{fig:cp_bound}, the residual errors (blue) are effectively bounded by the adaptive CP with 90\% confidence (orange), whereas the static CP bound computed from offline data (dashed green) does not adapt to disturbance variance changes. The bottom subplot illustrates how the estimated bound is used to construct upper and lower bounds on the disturbance estimate. The proposed method (red) produces a narrower and less conservative interval between the upper and lower bounds than the static method (green).

%  shows that:

% \begin{itemize}
%     \item The conformal bound adapts online to disturbance intensity.
%     \item Estimation error remains covered by the probabilistic bound.
%     \item The sliding-window mechanism avoids overly conservative constant bounds.
% \end{itemize}

% Necessary signal filtering and parameter tuning are applied to 
% ensure stable quantile estimation.

% This confirms that conformal uncertainty quantification can 
% effectively replace conservative deterministic disturbance bounds 
% in robust CBF design.

%%%%%%%%%%%%%%%%%%%%%%%%%%%%%%%%%%%%%%%%%%%%%%%%%%%%%%%%%%%%%%%%%%%%%%%%%%

%\vspace{-3mm}
\section{Conclusions}

% This paper presented a robust operational space control framework that unifies observer-based disturbance estimation and data-driven uncertainty quantification for safe redundant manipulation. By integrating an extended state observer (ESO) with conformal prediction, the proposed method achieves disturbance compensation and safety enforcement without requiring full-state measurements or conservative a priori disturbance bounds.

% The ESO provides structured dynamic disturbance estimation with explicit convergence guarantees, while the sliding-window conformal mechanism enables online probabilistic estimation of disturbance variation bounds. Embedding the resulting bound into a robust high-order control barrier function yields safety guarantees under both deterministic and probabilistic settings.

% Experimental validation on a 7-DoF Franka Research 3 manipulator demonstrated millimeter-level tracking accuracy and real-time safety enforcement at 1~kHz under significant dynamic disturbances, including sloshing payload injection. Compared to residual learning-based approaches, the proposed framework achieves improved robustness, reduced conservatism, and simplified design complexity.

% Future work will investigate extension to multi-constraint safety scenarios, contact-rich manipulation, and theoretical analysis of conformal bound tightness in closed-loop adaptive systems.
This paper presented a robust operational space control framework that integrates an extended state observer with sliding-window conformal prediction for safe redundant manipulation. The ESO provides disturbance compensation with convergence guarantees, while conformal prediction yields online probabilistic bounds on disturbance variation without conservative prior assumptions. Embedding these bounds into a robust high-order control barrier function ensures safety under uncertainty. Experiments on a 7-DoF Franka Research 3 manipulator demonstrate millimeter-level tracking and real-time (1 kHz) safety enforcement under significant dynamic disturbances, outperforming residual learning approaches in robustness, conservatism, and design simplicity.

%\addtolength{\textheight}{-12cm}   % This command serves to balance the column lengths
                                  % on the last page of the document manually. It shortens
                                  % the textheight of the last page by a suitable amount.
                                  % This command does not take effect until the next page
                                  % so it should come on the page before the last. Make
                                  % sure that you do not shorten the textheight too much.

%%%%%%%%%%%%%%%%%%%%%%%%%%%%%%%%%%%%%%%%%%%%%%%%%%%%%%%%%%%%%%%%%%%%%%%%%%%%%%%%

%%%%%%%%%%%%%%%%%%%%%%%%%%%%%%%%%%%%%%%%%%%%%%%%%%%%%%%%%%%%%%%%%%%%%%%%%%%%%%%%

%%%%%%%%%%%%%%%%%%%%%%%%%%%%%%%%%%%%%%%%%%%%%%%%%%%%%%%%%%%%%%%%%%%%%%%%%%%%%%%%
%\section*{APPENDIX}

%....

%\section*{ACKNOWLEDGMENT}

%....

%%%%%%%%%%%%%%%%%%%%%%%%%%%%%%%%%%%%%%%%%%%%%%%%%%%%%%%%%%%%%%%%%%%%%%%%%%%%%%%%

\bibliographystyle{IEEEtran}
\bibliography{Reference}

\end{document}